\documentclass{article}

\usepackage{nips15submit_e,times}
\nipsfinalcopy

\usepackage[utf8]{inputenc}
\usepackage[T1]{fontenc}
\usepackage{hyperref}
\usepackage{url}
\usepackage{booktabs}
\usepackage{amsfonts,amsmath,amssymb,amsthm}
\usepackage{bm}
\usepackage{nicefrac}
\usepackage[expansion=false,protrusion=false]{microtype}
\usepackage{xcolor}
\usepackage{graphicx}
\usepackage{multirow}
\usepackage{float}
\usepackage{subcaption}
\usepackage[numbers]{natbib}

\newtheorem{theorem}{Theorem}
\newtheorem{lemma}{Lemma}

\theoremstyle{definition}

\theoremstyle{remark}

\newcommand{\bbR}{\mathbb{R}}
\newcommand{\bbS}{\mathbb{S}}
\newcommand{\bbB}{\mathbb{B}}
\newcommand{\bbZ}{\mathbb{Z}}

\newcommand{\half}{\tfrac{1}{2}}

\title{FibQuant: Universal Vector Quantization for Random-Access 
KV-Cache Compression}

\author{Namyoon Lee \\
POSTECH \\
\texttt{nylee@postech.ac.kr}
\And
Yongjune Kim \\
POSTECH \\
\texttt{yongjune@postech.ac.kr}}

\begin{document}
\maketitle

%======================================================================
\begin{abstract}
Long-context inference is increasingly a memory-traffic problem.
The culprit is the key--value (KV) cache: it grows with context
length, batch size, layers, and heads, and it is read at every
decoding step.  A useful cache codec is therefore not merely a
quantizer.  It is also a data structure: every cached key and value
must occupy a fixed address and be reconstructed independently, so
decompression can be fused with attention.  Rotation-based scalar
codecs meet this systems constraint by storing a norm, applying a
shared random rotation, and quantizing one coordinate at a time.
They are universal and random-access, but they discard the geometry
created by the normalization step.  After a Haar rotation, a block
of $k$ consecutive coordinates is not a product source; it is a
spherical-Beta source on the unit ball.  We introduce
\textsc{FibQuant}, a universal fixed-rate vector quantizer that
keeps the same normalize--rotate--store interface while replacing
scalar tables by a shared radial--angular codebook matched to this
canonical source.  The codebook combines Beta-quantile radii,
Fibonacci\,/\,Roberts--Kronecker quasi-uniform directions, and
multi-restart Lloyd--Max refinement.  We prove that the resulting
vector code strictly improves on its scalar product specialization
at matched rate, with a high-rate gain that separates into a
cell-shaping factor and a density-matching factor.  The same
construction gives a dense rate axis, including fractional-bit and
sub-one-bit operating points, without calibration or variable-length
addresses.  On GPT-2 small KV caches, \textsc{FibQuant} traces a
memory--fidelity frontier from $5\times$ compression at $0.99$
attention cosine similarity to $34\times$ at $0.95$.  End-to-end on
TinyLlama-1.1B, it is within $0.10$ perplexity of fp16 at $4\times$
compression and has $3.6\times$ lower perplexity than scalar
\textsc{TurboQuant} at $b = 2$ ($8\times$ compression), where scalar
random-access quantization begins to fail.
\end{abstract}

%======================================================================
\section{Introduction}
\label{sec:intro}

The key--value (KV) cache is the state carried by an autoregressive
transformer.  Its size grows linearly with context length, batch
size, layer count, and head count; its contents are accessed at
every decoding step.  In long-context serving, the cache is therefore
not an auxiliary buffer.  It is often the object that determines
memory bandwidth, batch size, and latency, and in extreme regimes it
can exceed the model weights itself~\citep{jiang2025kvcomp}.  This
pressure has led to two-bit and lower-precision cache quantization
methods such as \textsc{KIVI} and
\textsc{KVQuant}~\citep{liu2024kivi, hooper2024kvquant}.

The cache, however, imposes a constraint that is absent from
ordinary source coding.  A deployable cache codec must be
\emph{random-access}: the representation of token $t$ must live at
a fixed address, and the corresponding key or value must be decoded
independently by simple address arithmetic.  Otherwise decompression
cannot be fused with the attention kernel~\citep{jiang2025kvcomp}.
This single constraint explains why fixed-rate quantizers dominate
the source-coding layer of KV-cache
compression~\citep{liu2024kivi, hooper2024kvquant,
zandieh2025turboquant, gao2024rabitq}.  It rules out many
statistically attractive variable-length or context-dependent
representations before one even asks about distortion.

The cleanest representative of this random-access family is
\textsc{TurboQuant}~\citep{zandieh2025turboquant}.  It stores the
vector norm, applies a shared random orthogonal rotation, and
quantizes each rotated coordinate with the same scalar Lloyd--Max
table.  This is the right \emph{interface}: source-agnostic,
fixed-rate, and easy to fuse.  But it is the wrong \emph{geometry}.
Once a vector has been normalized and Haar-rotated, a block of $k$
consecutive coordinates lies on the unit ball with a specific radial
law and a uniform angular component.  The coordinates are not an
independent product of shifted-Beta marginals.  A scalar code sees
one coordinate at a time; the source seen by the cache is
intrinsically vectorial.

\paragraph{The sub-one-bit gap.}
In production, the admissible cache rate is not chosen by an elegant
integer-bit quantizer.  It is set by VRAM, latency, context length,
and batch size.  When this budget falls below roughly one bit per
coordinate, current universal, calibration-free, random-access
codecs cease to operate: \textsc{TurboQuant} bottoms out at one bit,
\textsc{RaBitQ}~\citep{gao2024rabitq} is a one-bit signed code, and
\textsc{KIVI}\,/\,\textsc{KVQuant} primarily target two bits or
above.  The remaining alternatives either rely on calibration, as in
low-rank projection methods such as \textsc{Palu}~\citep{chang2024palu},
or alter the sequence itself through eviction or retention policies
such as \textsc{StreamingLLM} and
\textsc{H2O}~\citep{xiao2024streamingllm, zhang2023h2o}.  The
missing object is a fixed-rate, calibration-free, random-access
codec with a continuously tunable rate axis below one bit per
coordinate.

\paragraph{Contribution.}
We ask a simple question: \emph{what is the canonical source induced
by the random-access normalize--rotate interface?}  The answer is
the spherical-Beta law $f_{d, k}$ on the unit ball $\bbB^k$.  Once
this law is identified, the rest of the codec is classical.  The
radius is matched by Bennett--Gersho companding, the direction by
quasi-uniform spherical point sets, and the finite-rate codebook by
Lloyd--Max refinement.  The result is \textsc{FibQuant}: a vector
quantizer that keeps the serving interface of scalar rotation codes
but recovers the shaping and density-matching gains that scalar
products leave on the table.

Our contributions are as follows.
\begin{itemize}
\item \textbf{Canonical source.}  We prove that
norm-then-Haar-rotation maps every nonzero input distribution to the
same $k$-block marginal $f_{d, k}$ on $\bbB^k$
(Theorem~\ref{thm:univ}).  Thus one shared offline codebook is the
correct universal object for layers, heads, prompts, and models.
\item \textbf{Vector advantage.}  We show that the best $k$-vector
code strictly dominates the scalar product specialization at matched
rate, and that the high-rate gain factors into a cell-shaping term
and a density-matching term (Theorem~\ref{thm:dominance}).
\item \textbf{Dense rate axis.}  Because the rate is
$b = (\log_2 N)/k$, \textsc{FibQuant} realizes fractional-bit and
sub-one-bit operating points while preserving fixed addresses and
independent decoding (Figure~\ref{fig:tq_vs_fq_d128}).
\item \textbf{KV-cache evidence.}  On GPT-2 small,
\textsc{FibQuant} reaches $34.1\times$ compression at $0.946$
attention-output cosine similarity, the only universal random-access
method in our comparison beyond $10\times$ compression
(Figure~\ref{fig:gpt2_pareto}).  On TinyLlama-1.1B, it strictly
dominates per-token \textsc{Int} and scalar \textsc{TurboQuant} at
matched integer rates on WikiText-103 perplexity and HellaSwag
accuracy (Sec.~\ref{sec:llm_rd}).
\end{itemize}

%======================================================================
\section{Background and Related Work}
\label{sec:related}

\textsc{FibQuant} is a source-coding layer for a specific systems
interface: fixed-rate random access to cached keys and values.  It
is therefore complementary to eviction, low-rank projection, and
fused-kernel systems work.  We separate the related work according
to this role.

\paragraph{KV-cache compression.}
Existing approaches reduce the KV cache along different axes.
Quantizers reduce precision, from outlier-aware matrix quantization
such as \textsc{LLM.int8()}~\citep{dettmers2022llmint8} to
cache-specific methods such as \textsc{KIVI}, \textsc{KVQuant},
\textsc{Gear}, and \textsc{Qjl}~\citep{liu2024kivi, hooper2024kvquant,
kang2024gear, zandieh2025qjl}.  Token-selection methods reduce the
sequence dimension by keeping attention sinks or heavy
hitters~\citep{xiao2024streamingllm, zhang2023h2o, li2024snapkv,
cai2024pyramidkv}.  Low-rank methods such as \textsc{Palu} exploit
calibrated head-specific subspaces~\citep{chang2024palu}.  System
frameworks such as \textsc{KVComp} focus on making compression
compatible with the memory hierarchy and the attention
kernel~\citep{jiang2025kvcomp}.  These methods attack different
bottlenecks.  Our question is narrower and more primitive:
\emph{given a random-access fixed-rate cache slot, what source code
should occupy it?}

\paragraph{Rotation-then-quantize codes.}
The closest prior art is the family of source-agnostic rotation
codes.  \textsc{TurboQuant} stores a norm, applies a shared random
rotation, and quantizes each rotated coordinate using a Lloyd--Max
table matched to the shifted-Beta scalar
marginal~\citep{zandieh2025turboquant}.  \textsc{RaBitQ} reaches an
analogous universality principle for approximate nearest-neighbor
search through a one-bit signed code with a sharp distance
bound~\citep{gao2024rabitq}; a recent comparison studies these two
views side by side~\citep{revisiting2026rabitq_tq}.  These methods
identify the correct deployment interface.  The difference is that
they stop at scalar marginals.  \textsc{FibQuant} applies the same
interface to $k$-blocks, where the induced law is the
spherical-Beta vector distribution rather than a product of scalar
marginals.  This single change yields a dense rate axis and the
vector-quantization gains proved in Sec.~\ref{sec:theory}.

\paragraph{Spherical vector quantization.}
The codebook geometry uses classical ingredients.  The planar
Fibonacci spiral at $k = 2$ is related to the Golden Quantizer for a
complex Gaussian source~\citep{larsson2018golden}, Fibonacci-lattice
color quantization~\citep{mojsilovic2001color}, and sunflower
antenna arrays~\citep{vigano2009sunflower}.  Higher-dimensional
directions use Fibonacci-sphere and Roberts--Kronecker
low-discrepancy
constructions~\citep{gonzalez2010fibonacci, saff1997sphere,
roberts2018quasirandom, kuipers1974uniform, niederreiter1992qmc},
and the refinement step is the Lloyd--Max\,/\,Gersho theory of
vector quantization~\citep{max1960quantizing, lloyd1982pcm,
gersho1979asymptotic, gray1998quantization, conway1999sphere}.  What
is new here is not any one geometric primitive.  It is the matching
of those primitives to the spherical-Beta source forced by
random-access KV-cache compression.

%======================================================================
\section{Model and Codec}
\label{sec:method}

\begin{figure}[t]
\centering
\includegraphics[width=0.81\linewidth]{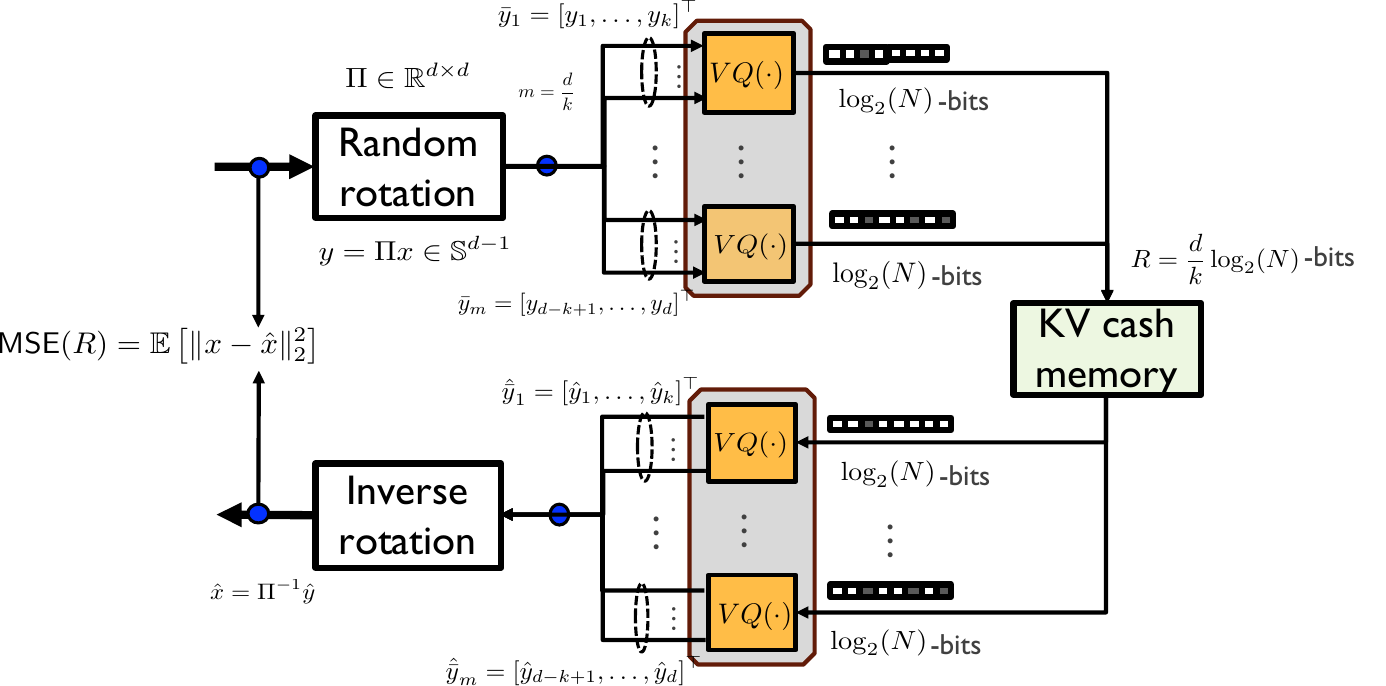}
\caption{\textbf{\textsc{FibQuant} encoder--decoder pipeline.}  A
cached vector $x \in \bbR^d$ is split into a scalar norm header
$\nu = \|x\|_2$ and a unit direction $\Pi x / \nu$, where
$\Pi \in \bbR^{d \times d}$ is a single Haar-random orthogonal
matrix shared across layers, heads, prompts and tokens.  The unit
direction is partitioned into $d/k$ blocks of $k$ consecutive
coordinates, each falling on the unit ball $\bbB^k$ with marginal
density $f_{d, k}$ (Lemma~\ref{lem:sphbeta}).  Each block is
encoded to one of $N$ indices by nearest-codeword lookup against
the shared offline codebook $\mathcal{C} = \{c_n\}_{n = 1}^{N}$
(Eq.~\eqref{eq:encode}); the decoder reverses the lookup and the
rotation to recover $\hat x$.  The output bitstream is a fixed-rate
$(d \log_2 N)/k$-bit payload plus one fp16 norm per cached vector,
so token slots are at affine offsets and any past key/value can be
recovered independently---the random-access requirement that
distinguishes deployable cache codecs from variable-length
alternatives.  The only change from \textsc{TurboQuant}'s pipeline
is the replacement of the per-coordinate scalar table by one
$k$-dimensional vector quantizer.}
\label{fig:system}
\end{figure}

\paragraph{Overview.}
The codec has three parts.  First, every cached vector is
normalized and rotated by a shared Haar-random orthogonal matrix.
This is the same random-access interface used by scalar rotation
codes: the norm is a small side header, the payload is fixed length,
and every token can be decoded independently.  Second, the rotated
unit vector is partitioned into $k$-dimensional blocks.
Lemma~\ref{lem:sphbeta} shows that each block has the same
spherical-Beta density $f_{d, k}$ on the unit ball, independent of
the original KV distribution.  Third, a shared offline codebook
$\mathcal{C} = \{r_n u_n\}_{n = 1}^{N}$ is built for this density:
Beta-quantile radii $r_n$, quasi-uniform directions $u_n$, and a
multi-restart Lloyd--Max polish.  The encoder performs
nearest-codeword lookup per block; the decoder performs table lookup
and inverse rotation.

The point is deliberately conservative.  \textsc{FibQuant} changes
the statistical object being quantized, not the serving contract.
The bitstream remains fixed-rate and randomly addressable; only the
scalar coordinate table is replaced by a vector codebook matched to
the law that the interface itself creates.

The section proceeds in three steps: the encoder--decoder pair
(Sec.~\ref{sec:rand_access}), the canonical source law
(Sec.~\ref{sec:canonical_source}), and the radial--angular codebook
construction (Sec.~\ref{sec:beta_quantile}).

\subsection{Random-access vector encoding}
\label{sec:rand_access}

Fix one attention head of width $d$, a shared orthogonal matrix
$\Pi \in \bbR^{d \times d}$, and a codebook
$\mathcal{C} = \{c_1, \ldots, c_N\} \subset \bbB^k$ with $k \mid d$.
For a nonzero cached vector $x \in \bbR^d$, the encoder stores the
fp16 norm header $\nu = \|x\|_2$ and quantizes the rotated unit
vector $\Pi x / \nu$.  Write its block decomposition as $\Pi x / \nu
= (y^{(1)}, \ldots, y^{(d/k)})$, with $y^{(m)} \in \bbB^k$.  The
stored index for block $m$ is
\begin{equation}
i_m \;=\; \arg\min_{1 \leq j \leq N} \|y^{(m)} - c_j\|_2^2,
\qquad m = 1, \ldots, d/k.
\label{eq:encode}
\end{equation}
The decoder sets $\hat y^{(m)} = c_{i_m}$ and returns
$\hat x = \nu \Pi^\top \hat y$.  The payload rate is
$b = (\log_2 N)/k$ bits per coordinate, identical for every vector
and every token.  Hence token addresses are affine in the time
index, and random access is preserved by construction.

\paragraph{Loss functions.}
We track two losses.  The intrinsic source-coding loss is
\begin{equation}
\mathcal{L}_{\mathrm{mse}}(\mathcal{C})
\;=\;
\mathbb{E}\sum_{m = 1}^{d/k}\,\min_{j} \|Y^{(m)} - c_j\|_2^2,
\end{equation}
where $Y^{(m)}$ is distributed according to the canonical law
$f_{d, k}$ derived in Sec.~\ref{sec:canonical_source}.  The
downstream metric is attention-output cosine similarity,
\begin{equation}
\mathcal{L}_{\mathrm{attn}}(\mathcal{C})
\;=\;
\mathbb{E}_{\ell, h, q}\,
\frac{\langle o_{\ell, h}(q),\, \hat o_{\ell, h}(q)\rangle}
     {\|o_{\ell, h}(q)\|\, \|\hat o_{\ell, h}(q)\|},
\qquad
o_{\ell, h}(q)
\;=\;
\mathrm{softmax}\!\left(\tfrac{q K_{\ell, h}^\top}{\sqrt d}\right) V_{\ell, h},
\label{eq:lattn}
\end{equation}
with $\hat o_{\ell, h}$ computed from the reconstructed
$\hat K_{\ell, h}, \hat V_{\ell, h}$.  A task-weighted or
Mahalanobis Lloyd objective can be used offline by replacing
$\|x - c\|_2^2$ with $(x - c)^\top M_h (x - c)$; the random-access
bitstream and decoder are unchanged.

\subsection{The canonical vector source}
\label{sec:canonical_source}

The normalization removes scale; the Haar rotation removes
orientation.  Consequently the distribution seen by the block
quantizer is universal.  Let $\bar x = x / \|x\|$.  Conditional on
any fixed direction $\bar x = v$, $\Pi v$ is Haar-uniform on the
unit sphere $\bbS^{d - 1}$ by rotational invariance.  Averaging over
$v$ does not change this law.  Thus every input distribution with
$\Pr(X = 0) = 0$ induces the same rotated unit-vector law.  The
$k$-block marginal is the following spherical-Beta density.

\begin{lemma}[Spherical-Beta vector marginal]
\label{lem:sphbeta}
Let $U \sim \mathrm{Unif}(\bbS^{d - 1})$ and $X = U_{1:k}$ for
$1 \leq k < d$.  Then $X$ has density
$f_{d, k}(x) = C_{d, k}(1 - \|x\|^2)^{(d - k - 2)/2}\,\mathbf{1}\{\|x\|
\leq 1\}$ on $\bbB^k$, with
\begin{equation}
R^2 \sim \mathrm{Beta}\!\left(k/2,\,(d - k)/2\right),
\quad
\mathbb{E} R^2 = k/d,
\quad
\mathrm{Var}(R^2) = \frac{2 k(d - k)}{d^2 (d + 2)},
\label{eq:radial_law}
\end{equation}
and $X / R$ uniform on $\bbS^{k - 1}$, independent of $R$.
\end{lemma}

At $k = 1$, Lemma~\ref{lem:sphbeta} reduces to the shifted-Beta
scalar marginal used by scalar \textsc{TurboQuant}.  For $k > 1$,
however, the joint law is not the product of those marginals.  This
distinction is the whole source of the vector gain: scalar
quantization is matched to one-coordinate projections, while the
cache interface naturally exposes blocks with radial--angular
structure.

Two consequences guide the design.
\begin{itemize}
\item[\textbf{(P1)}] \emph{Stationarity after rotation.} Every block
$Y^{(m)}$ has density $f_{d, k}$ on $\bbB^k$, independent of layer,
head, prompt, token, and the original KV distribution.  A single
shared codebook is therefore a universal offline object.
\item[\textbf{(P2)}] \emph{Concentration on a shell.} Since
$\mathrm{Var}(R^2) = O(d^{-2})$, high-dimensional mass concentrates
near $\bar R = \sqrt{k/d}$.  At large $d$, the problem becomes
primarily one of placing directions on $\bbS^{k - 1}$; at moderate
$d$, the residual radial spread is captured by Beta-quantile radii.
\end{itemize}

\subsection{Radial--angular codebook construction}
\label{sec:beta_quantile}

Lemma~\ref{lem:sphbeta} factorizes the design problem.  Each
codeword is written as $c_n = r_n u_n$, with radius $r_n \in [0, 1]$
and direction $u_n \in \bbS^{k - 1}$.  The radii match the radial
density of the spherical-Beta source; the directions approximate a
uniform spherical packing; Lloyd--Max refinement then adapts the
deterministic initialization to the finite value of $N$.

\paragraph{Radii.}
Bennett's high-rate companding rule, in its vector-quantization
form due to Gersho~\citep{gersho1979asymptotic}, places codewords
with density proportional to $f_{d, k}^{\,k/(k + 2)}$ for MSE.
Specializing this rule to the radius gives a Beta law with shape
$(k/2, \beta_{d, k})$.  Midpoint quantiles $q_n = (n - \half)/N$
yield
\begin{equation}
r_n \;=\; \sqrt{\mathrm{BetaInv}\!\left(q_n;\;\tfrac{k}{2},\;\beta_{d, k}\right)},
\qquad
\beta_{d, k} \;=\; \frac{k}{k + 2}\cdot\frac{d - k - 2}{2} + 1.
\label{eq:radii_main}
\end{equation}
For $k = 2$, this inverse is closed form:
\begin{equation}
r_n \;=\; \sqrt{1 - (1 - q_n)^{4/d}}.
\label{eq:radii_k2}
\end{equation}

\paragraph{Directions.}
The direction set $\{u_n\}_{n = 1}^{N} \subset \bbS^{k - 1}$ is
chosen by geometry, not by calibration data:
\begin{itemize}
\item $k = 2$: a planar Fibonacci spiral, $u_n = (\cos\theta_n,
\sin\theta_n)$ with $\theta_n = 2\pi(n - 1)\theta_g$ and
$\theta_g = 1 - 1/\varphi$.
\item $k = 3$: a Fibonacci sphere with equal-area latitude bands
and golden-angle azimuth; the multi-shell variant is given in
Appendix~\ref{app:k3_construction}.
\item $k \geq 4$: a Roberts--Kronecker rank-one sequence
$\xi_{n, j} = \{(n - \half)\, \phi_k^{-j}\}$, where $\phi_k$ is the
positive root of $\phi^{k + 1} = \phi + 1$, mapped through
$\Phi^{-1}$ and projected onto $\bbS^{k - 1}$.
\end{itemize}

\paragraph{Lloyd--Max refinement.}
The deterministic initialization $c_n^{(0)} = r_n u_n$ is refined on
samples from $f_{d, k}$.  Each restart applies a random orthogonal
rotation to the initial codebook, alternates nearest-neighbor
assignment and centroid update, repairs empty cells by splitting
high-distortion cells, and keeps the codebook with the lowest
training MSE.  The full procedure is specified in
Appendix~\ref{app:algorithm}, Table~\ref{alg:codebook}.  This
refinement is offline metadata: it does not alter the fixed-rate
payload or the random-access decoder.

\subsection{Example: \texorpdfstring{$k = 2$}{k=2}}
\label{sec:k2_example}

The two-dimensional case shows the construction in closed form.
Here $\beta_{d, 2} = d/4$, so~\eqref{eq:radii_main} becomes
\eqref{eq:radii_k2}.  Combining these radii with the golden-angle
sequence gives the sunflower initialization
$c_n^{(0)} = r_n (\cos\theta_n, \sin\theta_n)$, $n = 1, \ldots, N$,
followed by the same Lloyd--Max polish used in all dimensions.

\begin{figure}[t]
\centering
\includegraphics[width=0.81\linewidth]{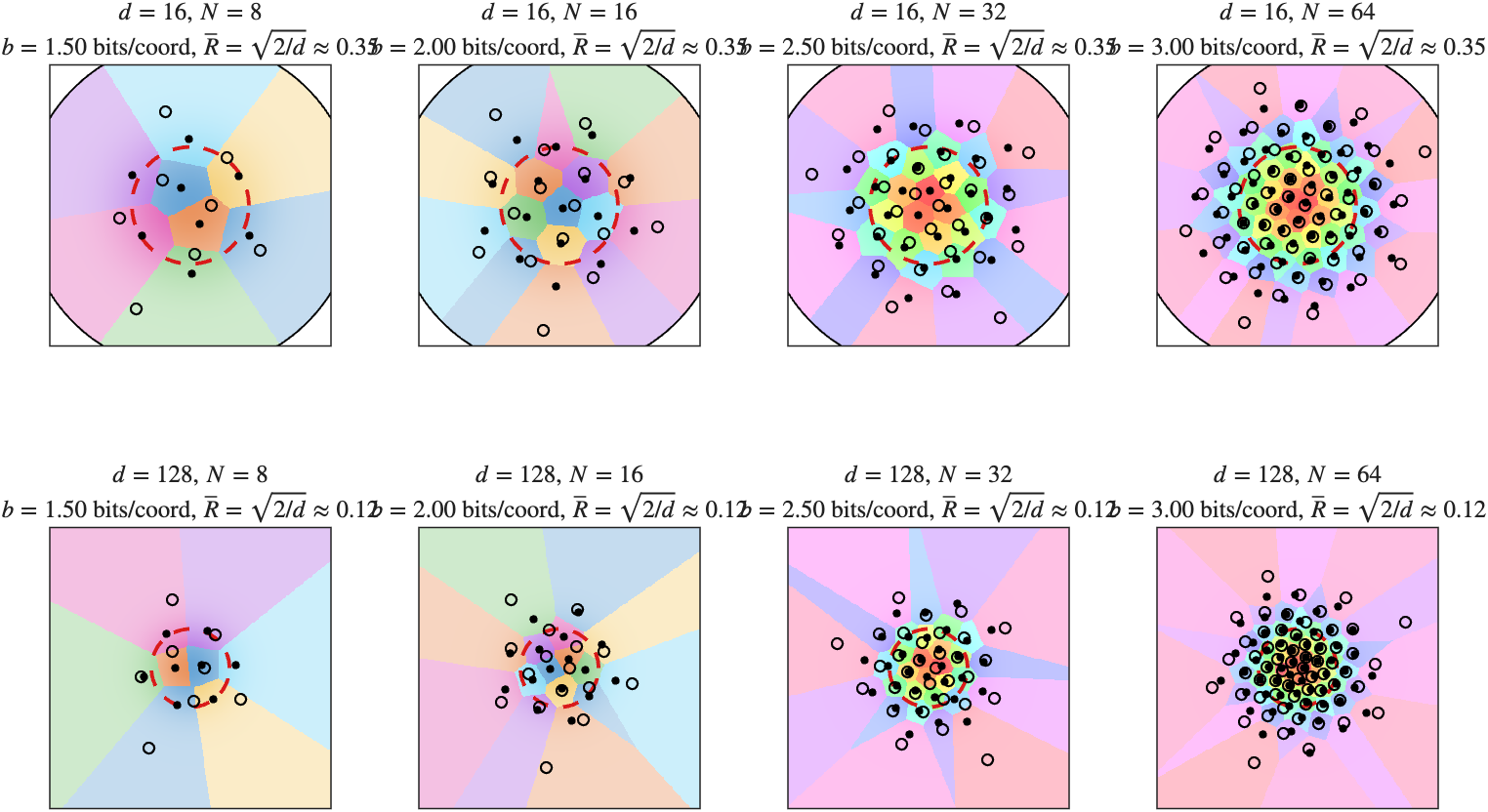}
\caption{\textsc{FibQuant} $k = 2$ codebooks and their Voronoi
cells on $\bbB^2$ at $d = 16$ (top row) and $d = 128$ (bottom row),
for $N \in \{8, 16, 32, 64\}$.  Open markers: Beta-quantile +
Fibonacci init.  Solid markers: codewords after Lloyd--Max.
Background heat-map: spherical-Beta source density $f_{d, 2}$.
Discussion in Sec.~\ref{sec:k2_example}.}
\label{fig:voronoi_k2}
\end{figure}

Figure~\ref{fig:voronoi_k2} shows the result for
$N \in \{8, 16, 32, 64\}$ and $d \in \{16, 128\}$.  At $d = 16$, the
spherical-Beta density has visible radial spread, and the codewords
use several radii.  At $d = 128$, the mass has nearly collapsed
onto the shell $\bar R = \sqrt{2/d} \approx 0.125$, and the problem
becomes mostly angular.  Along that shell, the Lloyd-polished cells
approach the hexagonal geometry expected of an efficient
two-dimensional vector quantizer.  This picture is the
finite-dimensional version of Theorem~\ref{thm:dominance}: a scalar
product code can tile the plane by rectangles, but it cannot recover
the shaping gain of these cells.

The same effect appears quantitatively in
Figure~\ref{fig:tq_vs_fq_d128}.  At integer rates, the $k = 2$
branch lies below scalar \textsc{TurboQuant}; at fractional rates,
the same codebook family gives operating points that scalar
Lloyd--Max cannot realize at all.

%======================================================================
\section{Analysis}
\label{sec:theory}

The analysis formalizes the message of the construction.  Random
access fixes the source; the source fixes the correct codebook;
vector quantization then improves on the scalar product restriction.

\begin{theorem}[Source-agnostic universality]
\label{thm:univ}
Let $X$ be any random vector in $\bbR^d$ with $\Pr(X = 0) = 0$,
independent of a Haar-uniform orthogonal $\Pi$, and let
$Y = \Pi X / \|X\|$.  Then $Y$ is Haar-uniform on $\bbS^{d - 1}$,
and every $k$-coordinate vector of $Y$ has marginal density
$f_{d, k}$.
\end{theorem}

\begin{proof}
Condition on $X / \|X\| = v$.  By invariance of Haar measure,
$\Pi v$ is uniform on $\bbS^{d - 1}$.  Averaging over $v$ leaves the
Haar law unchanged.  Lemma~\ref{lem:sphbeta} gives the
$k$-coordinate marginal.
\end{proof}

The theorem is the universality statement.  It says that the shared
codebook is not a heuristic trained on a convenient model; it is the
canonical source codebook induced by the serving interface itself.

\begin{theorem}[Strict matched-rate dominance]
\label{thm:dominance}
Let $D_{\mathrm{sc}}(d, b)$ be the per-coordinate distortion of
scalar Lloyd--Max quantization of the shifted-Beta marginal
$f_{d, 1}$ at $b$ bits per coordinate, and let
$D_{\mathrm{FQ}}^\star(d, k, b)$ be the per-coordinate MSE of the
best $k$-vector code with $N = 2^{bk}$ codewords for $f_{d, k}$.
Then
\begin{equation}
D_{\mathrm{FQ}}^\star(d, k, b) \;<\; D_{\mathrm{sc}}(d, b),
\quad k > 1, \;b > 0,
\label{eq:strict-dom}
\end{equation}
with equality only in the scalar case $k = 1$.  At high rate,
\begin{equation}
\frac{D_{\mathrm{sc}}(d, b)}{D_{\mathrm{FQ}}^\star(d, k, b)}
\;\simeq\;
\underbrace{\frac{k\,G_1}{G_k^\star}}_{\gamma_{\mathrm{cell}}(k)\,\geq\,1}
\;\cdot\;
\underbrace{\frac{I_{d, 1}(\nicefrac{1}{3})^3}
{I_{d, k}(\nicefrac{k}{k + 2})^{(k + 2)/k}}}_{\gamma_{\mathrm{dens}}(d, k)\,\geq\,1},
\label{eq:hr_decomp}
\end{equation}
where $G_1 = 1/12$, $G_k^\star$ is the optimal $k$-cell normalized
second moment, and $I_{d, k}(s) = \int_{\bbB^k} f_{d, k}^{\,s}\,dx$.
The cell factor is strict for $k \geq
2$~\citep{conway1999sphere, gersho1979asymptotic}.
\end{theorem}

A scalar product code is a feasible $k$-vector code whose codebook
is a Cartesian product of one-dimensional levels.  The
unrestricted vector code optimizes over a larger class of
tessellations, and therefore cannot be worse.
Equation~\eqref{eq:hr_decomp} separates the gain into two terms.
The first is geometric: efficient vector cells beat rectangular
boxes.  The second is statistical: matching the spherical-Beta
density beats matching only its scalar marginals.  Numerically, the
cell-shaping term alone gives approximately $0.17$\,dB at $k = 2$,
$0.66$\,dB at $k = 3$, $1.42$\,dB at $k = 8$, and approaches
$10\log_{10}(\pi e/6) \approx 1.53$\,dB as $k \to \infty$.  The
density-matching term supplies the additional gain observed in
Fig.~\ref{fig:tq_vs_fq_d128}.

Full details, including the proof of Lemma~\ref{lem:sphbeta} and a
finite-rate large-$d$ characterization, are given in
Appendices~\ref{app:proofs}--\ref{app:large_d_RD}.

\paragraph{Rate granularity.}
\label{sec:rate_granularity}
Scalar Lloyd--Max tables offer integer-bit operating points.  In
contrast,
\begin{equation}
\mathcal{B}_{\mathrm{FQ}}
\;=\;
\{(\log_2 N)/k\,:\, k, N \in \bbZ\}
\end{equation}
is dense in $(0, \infty)$.  Thus the rate can be selected after the
serving budget is known, including sub-one-bit and fractional-bit
regimes, without changing the random-access interface.

%======================================================================
\section{Experiments}
\label{sec:exp}

The experiments test three questions.  First, does the vector
geometry predicted by Theorem~\ref{thm:dominance} reduce source
distortion on the canonical spherical-Beta law?  Second, does this
reduction survive in a real KV cache, where the relevant metric is
attention-output fidelity rather than raw MSE?  Third, does it
improve end-to-end language-model performance at matched cache
rates?

\subsection{Per-coordinate MSE on the canonical source}
\label{sec:mse_canonical}

We sample from $f_{d, k}$ for $d \in \{64, 256\}$ and
$k \in \{2, 4, 8, 16, 32\}$.  Each \textsc{FibQuant} codebook is
built by Table~\ref{alg:codebook} with $M = 30 N$ training samples,
$R = 4$ random restarts, and $T_{\mathrm{LM}} = 25$ Lloyd--Max
iterations per restart.  The scalar baseline is Lloyd--Max
quantization of the shifted-Beta marginal with $L = 2^b$ levels.

\begin{figure}[t]
\centering
\includegraphics[width=0.66\linewidth]{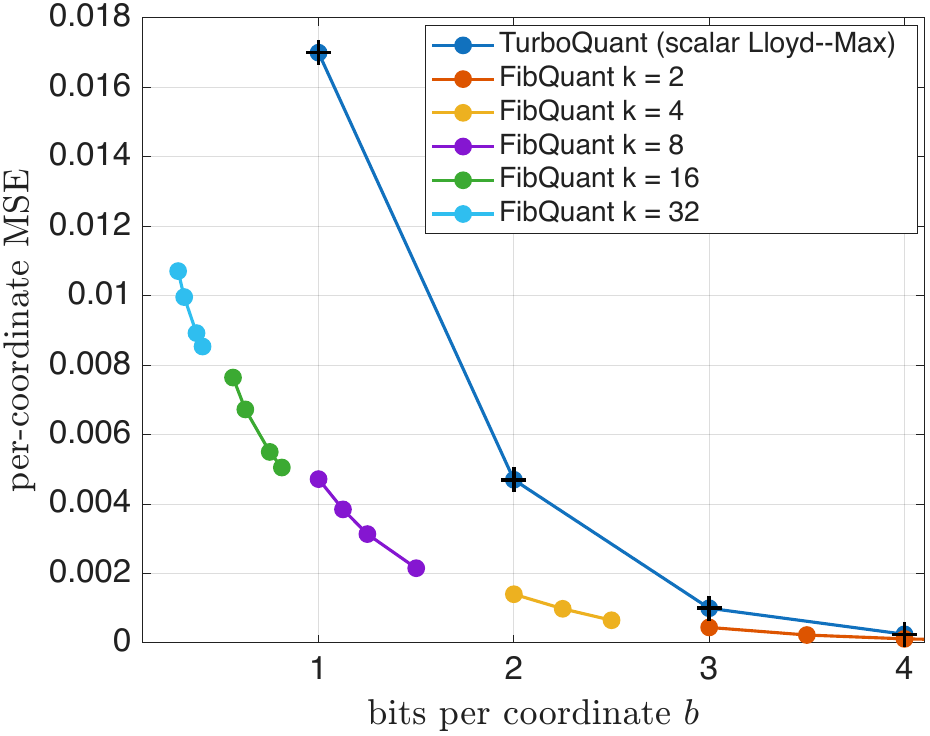}
\caption{\textbf{Per-coordinate MSE versus rate
$b = (\log_2 N)/k$ on the canonical spherical-Beta source
$f_{d, k}$ at $d = 64$.}  Scalar \textsc{TurboQuant} ($\bm{+}$)
contributes only the four integer-bit operating points
$b \in \{1, 2, 3, 4\}$, whereas \textsc{FibQuant} for
$k \in \{2, 4, 8, 16, 32\}$ traces a continuum across the rate
axis.  Three features are visible.  \emph{(i) Strict matched-rate
dominance.}  At every integer rate the \textsc{FibQuant} curve lies
strictly below the scalar curve; the vertical gap is the
$\gamma_{\mathrm{cell}}(k)\!\cdot\!\gamma_{\mathrm{dens}}(d, k)$
factorization of Theorem~\ref{thm:dominance},
Eq.~\eqref{eq:hr_decomp}.  \emph{(ii) Fractional-integer-bit
coverage.}  Between every pair of adjacent integer scalar points
\textsc{FibQuant} fills the band $b \in \{1.25, 1.5, \ldots, 3.5\}$
(annotated arrows) --- operating points that no scalar product code
can reach.  \emph{(iii) Sub-one-bit regime.}  Below $b = 1$
\textsc{FibQuant} alone operates, all the way down to
$b \approx 0.19$ at $(k, N) = (32, 64)$.  The $d = 256$ check
(omitted here for space) confirms the $1/d$ source-variance
scaling of Appendix~\ref{app:asymptotic}.}
\label{fig:tq_vs_fq_d128}
\end{figure}

Figure~\ref{fig:tq_vs_fq_d128} shows the expected separation.  At
the integer rates where scalar \textsc{TurboQuant} exists,
\textsc{FibQuant} lies below it.  Between those integer rates,
\textsc{FibQuant} fills the missing fractional-bit operating
points.  Below one bit per coordinate, scalar random-access
quantization has no corresponding point, while vector blocks
continue to trade codebook size against rate.  Repeating the
experiment at $d = 256$ shifts the curves down by the predicted
$1/d$ variance scaling.

\subsection{Real GPT-2 KV cache}
\label{sec:gpt2}

We next evaluate the full KV cache of GPT-2 small ($L = 12$,
$H = 12$, $T_{\mathrm{seq}} = 512$, $d = 64$; $18.87$\,MB per
sequence at fp16) on $16$ WikiText prompts.  The comparison
includes per-token \textsc{Int}~$\{2, 4, 8\}$, \textsc{KIVI},
\textsc{StreamingLLM}, \textsc{H2O}, low-rank SVD, scalar
\textsc{TurboQuant}, and \textsc{FibQuant}.  For each
reconstructed cache $\hat K, \hat V$, we run the original softmax
attention using $32$ random queries and average the per-head output
cosine similarity.  Memory counts charge all per-vector side
information, such as norms and \textsc{KIVI} scales, but amortize
shared codebooks and rotations across cached sequences.

\begin{figure}[t]
\centering
\includegraphics[width=0.70\linewidth]{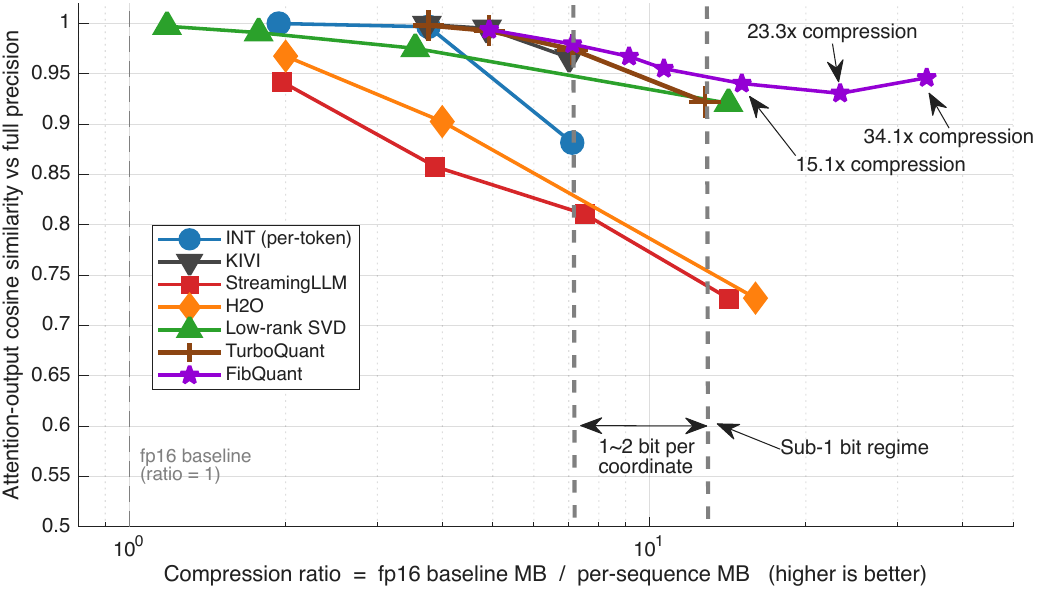}
\caption{Memory--fidelity Pareto on the real GPT-2 small KV cache
(one sequence, $T_{\mathrm{seq}} = 512$).  Horizontal axis:
compression ratio $=$ (fp16\,MB)/(per-seq\,MB) on a log scale,
higher is better.  Three guideline ratios are marked: $1\times$
(fp16 baseline), $10.7\times$ (last sub-2-MB scalar-equivalent
point), and $34.1\times$ (\textsc{FibQuant}'s extreme point at
$k = 64$, $N = 16384$).  Above $\sim 5\times$, \textsc{FibQuant},
\textsc{TurboQuant} and \textsc{KIVI} sit on top of one another.
Above $10\times$, every non-rotation universal method drops out,
and the curve splits into low-rank SVD (calibrated, per-(layer,
head)) and \textsc{FibQuant} (universal, calibration-free).}
\label{fig:gpt2_pareto}
\end{figure}

Figure~\ref{fig:gpt2_pareto} has three regimes.  Up to about
$7\times$ compression, \textsc{FibQuant}, \textsc{TurboQuant}, and
\textsc{KIVI} all preserve attention outputs well; the choice is
mainly systems-driven.  Between $10\times$ and $25\times$, the
non-rotation universal baselines either have no operating point or
lose fidelity rapidly, while \textsc{FibQuant} continues along a
dense curve.  Beyond $20\times$, only low-rank SVD and
\textsc{FibQuant} remain competitive; the distinction is that SVD
is calibrated per layer and head, whereas \textsc{FibQuant} is
universal and fixed-rate.  The most compressed \textsc{FibQuant}
point, $(k, N) = (64, 16384)$, reaches $34.1\times$ compression at
$0.946$ attention cosine similarity.

\paragraph{Why a full-vector code can win at a lower rate.}
At $N = 16384$, the $k = d = 64$ point compresses harder than the
$k = d/2 = 32$ point ($34.1\times$ versus $23.3\times$) and
nevertheless has higher attention fidelity ($0.946$ versus $0.931$).
The reason is geometric.  With $k = d$, the normalized vector lies
exactly on $\bbS^{d - 1}$, so all bits are spent on angular
packing---the same quantity read by attention inner products.  With
$k = d/2$, the reconstruction is a product of two independently
quantized halves and must also spend resolution on radial
fluctuations that are less aligned with attention fidelity.  The
per-vector reconstruction table in Appendix~\ref{app:recon} shows
the corresponding cosine improvement, which the softmax amplifies at
the attention-output level.

\subsection{End-to-end LLM rate--distortion}
\label{sec:llm_rd}

Finally, we replace the runtime KV cache of
TinyLlama-1.1B-Chat-v1.0 (Apache 2.0; $L = 22$, $H_{kv} = 4$,
$d = 64$) by $x \mapsto \mathrm{decode}(\mathrm{encode}(x))$.  The
model therefore uses exactly the reconstructed cache that would be
stored in memory.  We report sliding-window WikiText-103-raw
perplexity ($T_{\mathrm{seq}} = 8192$, window $2048$, stride $512$)
and zero-shot HellaSwag accuracy on $200$ items.
\textsc{FibQuant} is evaluated at $b \in \{2, 2.5, 3, 3.5, 4\}$;
\textsc{Int} and \textsc{TurboQuant} at $b \in \{2, 3, 4\}$.

\begin{table}[h]
\caption{End-to-end rate--distortion on TinyLlama-1.1B
($524{,}288$ WikiText-103 tokens, sliding window $2048$ / stride
$512$; $200$ HellaSwag items, log-likelihood ranking).  At every
matched integer rate \textsc{FibQuant} (bold) strictly dominates
both per-token \textsc{Int} and rotated-scalar \textsc{TurboQuant};
at fractional rates only \textsc{FibQuant} operates.  fp16
reference: $\mathrm{ppl} = 8.619$, HellaSwag $= 49.5\%$.}
\label{tab:llm_rd}
\begin{center}\small
% TinyLlama-1.1B-Chat-v1.0 (head_dim = 64, L = 22, H_kv = 4)
% WikiText-103 sliding-window perplexity, T_seq = 8192, window 2048, stride 512
% HellaSwag zero-shot, 200 validation items, log-likelihood ranking
% measured by experiments/eval_llama_fibquant.py
\begin{tabular}{l c c c c}
\toprule
Codec & $b$ (bits/coord) & Compression & WikiText-103 PPL $\downarrow$ & HellaSwag (\%) $\uparrow$ \\
\midrule
fp16 baseline                       & ---  & $1\times$  & $\mathbf{8.619}$ & $\mathbf{49.5}$ \\
\midrule
\textsc{INT}\,$b{=}4$                & $4.00$ & $4\times$  & $9.244$    & $47.5$ \\
\textsc{INT}\,$b{=}3$                & $3.00$ & $5.3\times$& $36.940$   & $37.0$ \\
\textsc{INT}\,$b{=}2$                & $2.00$ & $8\times$  & $4450.506$ & $22.5$ \\
%\textsc{INT}\,$b{=}1$                & $1.00$ & $16\times$ & $38559.7$  & $19.0$ \\
\midrule
\textsc{TurboQuant}\,$b{=}4$         & $4.00$ & $4\times$  & $9.220$    & $45.5$ \\
\textsc{TurboQuant}\,$b{=}3$         & $3.00$ & $5.3\times$& $11.413$   & $46.5$ \\
\textsc{TurboQuant}\,$b{=}2$         & $2.00$ & $8\times$  & $56.717$   & $33.5$ \\
%\textsc{TurboQuant}\,$b{=}1$         & $1.00$ & $16\times$ & $1130.400$ & $26.0$ \\
\midrule
\textsc{FibQuant} $k{=}2,\ N{=}256$  & $4.00$ & $4\times$  & $\mathbf{8.715}$  & $\mathbf{47.5}$ \\
\textsc{FibQuant} $k{=}2,\ N{=}128$  & $3.50$ & $4.6\times$& $\mathbf{8.881}$  & $\mathbf{48.0}$ \\
\textsc{FibQuant} $k{=}2,\ N{=}64$   & $3.00$ & $5.3\times$& $\mathbf{9.248}$  & $\mathbf{47.5}$ \\
\textsc{FibQuant} $k{=}4,\ N{=}1024$ & $2.50$ & $6.4\times$& $\mathbf{10.219}$ & $\mathbf{44.0}$ \\
\textsc{FibQuant} $k{=}4,\ N{=}256$  & $2.00$ & $8\times$  & $\mathbf{15.879}$ & $\mathbf{41.5}$ \\
%\textsc{FibQuant} $k{=}8,\ N{=}4096$ & $1.50$ & $10.7\times$& ---       & --- \\
%\textsc{FibQuant} $k{=}8,\ N{=}256$  & $1.00$ & $16\times$ & ---        & --- \\
%\textsc{FibQuant} $k{=}16, N{=}8192$ & $0.81$ & $19.7\times$& ---       & --- \\
%\textsc{FibQuant} $k{=}64, N{=}16384$& $0.22$ & $73\times$ & ---        & --- \\
\bottomrule
\end{tabular}

\end{center}
\end{table}

Table~\ref{tab:llm_rd} shows the same pattern end-to-end.  At
matched integer rates, \textsc{FibQuant} improves both perplexity
and HellaSwag.  The gap is modest at $b = 4$, larger at $b = 3$, and
decisive at $b = 2$: \textsc{FibQuant} obtains $15.88$ perplexity,
whereas scalar \textsc{TurboQuant} reaches $56.72$ and per-token
\textsc{Int} collapses.  The fractional points are equally
important.  At $b = 3.5$, \textsc{FibQuant} is within $0.27$
perplexity of fp16; at $b = 2.5$, it provides an intermediate
operating point unavailable to scalar fixed-bit codecs.  This is
the practical value of the dense rate axis: the cache budget can be
chosen first, and the codec can be placed exactly on it.

%======================================================================
\section{Discussion and conclusion}
\label{sec:disc}

\paragraph{Operating point and implementation.}
The encoder is a batched nearest-codeword search and the decoder is
a table lookup followed by inverse rotation.  The shared Pareto
library is small---about $5$\,MB in our implementation---and is
amortized across cached sequences.  The useful operating regions
are clear from Figure~\ref{fig:gpt2_pareto}: full-vector or
large-block codes are most attractive in the sub-one-bit regime;
$k \in \{4, 8\}$ is effective around one to two bits per
coordinate; and $k = 2$ is already sufficient in the higher-rate
regime.  If strict random access is relaxed, the emitted indices
could be entropy-coded using a
hyperprior~\citep{balle2018hyperprior}; that extension trades
addressability for additional rate reduction and is therefore
separate from the fixed-rate codec studied here.

\paragraph{Limitations.}
The theory identifies the canonical source and the high-rate vector
gain, but finite-rate Lloyd--Max optimality for every practical
$(d, k, N)$ is not characterized in closed form.  The largest
end-to-end experiment here uses TinyLlama-1.1B, and the GPT-2 study
reports attention-output fidelity rather than full generation
quality.  Larger models, longer contexts, bootstrap confidence
intervals over rotations and prompts, and fused-kernel
implementations are natural next steps.

\paragraph{Conclusion.}
The random-access constraint looks like an engineering detail, but
it has a mathematical consequence: after normalization and Haar
rotation, every cache block has the spherical-Beta law $f_{d, k}$.
Once that source is named, the design becomes classical.  Use
companding for the radius, quasi-uniform spherical points for the
direction, and Lloyd--Max to polish the finite-rate cells.  The
result keeps the fixed-address interface of scalar rotation codes,
recovers vector shaping and density-matching gains, and opens the
sub-one-bit regime without calibration.  Random access does not
preclude good source coding; it tells us which source to code.

%----------------------------------------------------------------------

{\small
\bibliographystyle{plain}

}

%======================================================================
\appendix

\section{Codebook construction algorithm}
\label{app:algorithm}

Table~\ref{alg:codebook} gives the full pseudocode for the
\textsc{FibQuant} codebook construction described in
Sec.~\ref{sec:beta_quantile}: Beta-quantile radii, quasi-uniform
directions, and multi-restart Lloyd--Max polish, together with the
large-$d$ shortcut from Appendix~\ref{app:asymptotic} that replaces
the radial law by the typical shell.

\begin{table}[h]
\caption{\textsc{FibQuant} codebook construction (Beta-quantile
radii, quasi-uniform directions, Lloyd--Max polish).}
\label{alg:codebook}
\begin{center}\small
\begin{tabular}{l}
\toprule
\textbf{Input:} block size $k$, codebook size $N$, ambient
dimension $d$, \\
\quad training set size $M$, restarts $R$, Lloyd--Max iterations
$T_{\mathrm{LM}}$. \\
\textbf{Output:} refined codebook
$\mathcal{C}^\star \subset \bbB^k$. \\
\midrule
1: Set $\beta_{d, k} \leftarrow \tfrac{k}{k + 2}\cdot\tfrac{d - k - 2}{2} + 1$. \\
2: \textbf{for} $n = 1, \ldots, N$ \textbf{do} \\
3: \quad $q_n \leftarrow (n - \half)/N$. \\
4: \quad \textbf{if} $k = 2$ \textbf{then}
$r_n \leftarrow \sqrt{1 - (1 - q_n)^{4/d}}$ \quad // closed form \\
5: \quad \textbf{else}
$r_n \leftarrow \sqrt{\mathrm{BetaInv}(q_n;\,k/2,\,\beta_{d, k})}$. \\
6: \textbf{end for} \\
7: Generate quasi-uniform directions $\{u_n\}_{n = 1}^{N}$ on
$\bbS^{k - 1}$ \\
\quad (Fibonacci spiral, Fibonacci sphere, or Roberts--Kronecker). \\
8: Initialize $\mathcal{C}^{(0)} \leftarrow \{r_n u_n\}_{n = 1}^{N}$. \\
9: Sample $X = \{x_i\}_{i = 1}^{M}$ from $f_{d, k}$ (or collect
calibrated rotated KV blocks). \\
10: $\mathcal{C}^\star \leftarrow \mathcal{C}^{(0)}$,
$D^\star \leftarrow \infty$. \\
11: \textbf{for} $j = 1, \ldots, R$ \textbf{do} \quad // multi-restart \\
12: \quad Rotate $\mathcal{C}^{(0)}$ by a random orthogonal
$\Omega_j \in \bbR^{k \times k}$ to get $\mathcal{C}$. \\
13: \quad \textbf{for} $t = 1, \ldots, T_{\mathrm{LM}}$ \textbf{do} \\
14: \qquad Assign each $x_i$ to its nearest codeword in $\mathcal{C}$; \\
15: \qquad update each occupied codeword to the assigned centroid; \\
16: \qquad repair empty cells by splitting a high-distortion donor. \\
17: \quad \textbf{end for} \\
18: \quad \textbf{if} $D(\mathcal{C}) < D^\star$ \textbf{then}
$\mathcal{C}^\star \leftarrow \mathcal{C},\; D^\star \leftarrow D(\mathcal{C})$. \\
19: \textbf{end for} \\
20: \textbf{return} $\mathcal{C}^\star$. \\
\midrule
\emph{Large-$d$ shortcut} (Appendix~\ref{app:asymptotic}): replace
lines 1--6 by \\
$r_n \leftarrow \bar R = \sqrt{k/d}$ for all $n$, then continue from
line 7. \\
\bottomrule
\end{tabular}
\end{center}
\end{table}

%----------------------------------------------------------------------
\section{Proof of Lemma~\ref{lem:sphbeta} and Theorem~\ref{thm:dominance}}
\label{app:proofs}

\paragraph{Spherical-Beta vector marginal (Lemma~\ref{lem:sphbeta}).}
Let $G \sim \mathcal{N}(0, I_d)$ and $U = G / \|G\|$, Haar-uniform
on $\bbS^{d - 1}$.  At any $x \in \bbB^k$, the remaining $d - k$
coordinates lie on a sphere of radius $\sqrt{1 - \|x\|^2}$ in
dimension $d - k - 1$, with surface area
$\propto (1 - \|x\|^2)^{(d - k - 2)/2}$.  The polar normalization
$\int_{\bbB^k} (1 - \|x\|^2)^{\alpha}\, dx
= \pi^{k/2}\,\Gamma(\alpha + 1)/\Gamma(\alpha + k/2 + 1)$ with
$\alpha = (d - k - 2)/2$ gives the constant $C_{d, k}$.  In polar
coordinates $x = r u$ the angular measure factors out and $r$ has
density $\propto r^{k - 1}(1 - r^2)^{(d - k)/2 - 1}$; the change of
variables $z = r^2$ yields
$R^2 \sim \mathrm{Beta}(k/2, (d - k)/2)$.  Specializing to $k = 1$
recovers the shifted-Beta coordinate marginal
$f_{d, 1}(x) \propto (1 - x^2)^{(d - 3)/2}$ used by scalar
\textsc{TurboQuant}, with $\mathbb{E} X = 0$ and
$\mathrm{Var}(X) = 1/d$.

\paragraph{Vector-code dominance (Theorem~\ref{thm:dominance}).}
A scalar product code on $k$ coordinates is a $k$-vector quantizer
whose codebook is the Cartesian product $\mathcal{A}^k$ of
$L_{\mathrm{sc}}$ scalar levels; the optimal $k$-vector quantizer
minimizes MSE over a strictly larger class, so
$D^\star_{\mathrm{FQ}} \leq D_{\mathrm{sc}}$, with equality iff the
optimum \emph{is} a product (only at $k = 1$).  For the high-rate
decomposition, Bennett's distortion integral gives
$D_{\mathrm{sc}}(d, b) \sim G_1\, 2^{-2b}\, I_{d, 1}(\nicefrac{1}{3})^3$
per coordinate, hence $k\,G_1\, 2^{-2b}\,
I_{d, 1}(\nicefrac{1}{3})^3$ for $k$ coordinates.  Gersho's
high-rate VQ result~\citep{gersho1979asymptotic} gives
$D^{\mathrm{vec}}_{\mathrm{FQ}}(d, k, b) \sim G_k^\star\, 2^{-2b}\,
I_{d, k}(\nicefrac{k}{k + 2})^{(k + 2)/k}$ per vector; dividing by
$k$ and taking the ratio yields~\eqref{eq:hr_decomp}.  Both factors
are $\geq 1$.

%----------------------------------------------------------------------
\section{Large-\texorpdfstring{$d$}{d} finite-rate behavior}
\label{app:large_d_RD}

The high-rate decomposition takes $N \to \infty$ with $(d, k)$
fixed; in serving, $d$ is also large.

\begin{theorem}[Large-$d$ finite-rate RD]
\label{thm:large_d}
Let $X = R\,U$ as in Lemma~\ref{lem:sphbeta}.  (i) Fixed $k$,
$d \to \infty$: $\sqrt d\, X \Rightarrow Z \sim \mathcal{N}(0, I_k)$,
and for every fixed $N$,
$d \cdot D^\star_{d, k, N} \to
\inf_{|\mathcal{C}| \leq N} \mathbb{E}_Z \min_{c \in \mathcal{C}} \|Z - c\|^2$.
(ii) Proportional regime, $k/d \to \rho \in (0, 1)$: $R^2 \to \rho$
in probability, and for the fixed-shell codebook
$\mathcal{C}_N = \{\sqrt\rho\, v_i\}_{i = 1}^{N}$ with directions
$v_i \in \bbS^{k - 1}$, $D_{d, k, N}(\mathcal{C}_N)
= 2\rho\,\mathbb{E}_U\!\left[1 - \max_v \langle U, v\rangle\right]
+ o(1)$.  Under the equal-cap heuristic with $N = 2^{bk}$ and
$t_b = \sqrt{1 - 2^{-2b}}$, the asymptotic shell distortion is
$2\rho(1 - t_b)$ for fixed-radius caps and $\rho\, 2^{-2b}$ for
Lloyd-centroid caps.
\end{theorem}

The proof reduces to the central limit theorem in (i) and to weak
convergence of the radial CDF plus an equal-cap measure argument in
(ii).  Regime (i) cautions against asymptotic intuition for the
small vector sizes used in KV caches; regime (ii) confirms that
once $k$ scales with $d$, the only design problem is angular
packing on $\bbS^{k - 1}$, where Fibonacci-type direction sets are
natural.

%----------------------------------------------------------------------
\section{Large-\texorpdfstring{$d$}{d} simplification of the radial design}
\label{app:asymptotic}

The radial spread of $f_{d, k}$ has standard deviation
$\mathrm{sd}(R) \approx (2\sqrt d)^{-1}$ around the typical radius
$\bar R = \sqrt{k/d}$.  As $d$ grows the source concentrates on an
ever-thinner shell, and the Beta-quantile radial design collapses
to a single-radius design: the Lloyd-polished single-shell codebook
$\mathcal{C}_d^{\mathrm{shell}} = \{\bar R\, u_n\}_{n = 1}^{N}$ has
the same finite-rate optimum as the Beta-quantile codebook
$\mathcal{C}_d^{\mathrm{Beta}}$, with the difference vanishing at
$O(1/d)$ in normalized MSE.  Use~\eqref{eq:radii_main} at moderate
$d$; use the single-shell init at large $d$ with the same Lloyd
polish.

%----------------------------------------------------------------------
\section{Fibonacci sphere and the multi-shell codebook at \texorpdfstring{$k = 3$}{k=3}}
\label{app:k3_construction}

This appendix gives the angular construction at $k = 3$ in two
stages: \emph{single-shell} (uniform angular quantization at one
fixed radius) and \emph{multi-shell} (composing several Beta-quantile
radii with one Fibonacci-sphere direction set per shell).  Together
they specify the deterministic 3D codebook that \textsc{FibQuant}
Lloyd-polishes before deployment.

\paragraph{Single-shell Fibonacci sphere: uniform angular quantization.}
On the unit sphere $\bbS^{2}$ the Fibonacci sphere places $M_{a}$
codewords at equal-area latitude bands times golden-angle azimuth
\citep{fibonacci_lattices_observable, gonzalez2010fibonacci}:
\begin{equation}
z_n = 1 - \tfrac{2 n - 1}{M_{a}},
\qquad
\theta_n = 2\pi(n - 1)(1 - 1/\varphi),
\qquad
u_n = \bigl(\sqrt{1 - z_n^2}\,\cos\theta_n,\ \sqrt{1 - z_n^2}\,\sin\theta_n,\ z_n\bigr),
\label{eq:fib_sphere_3d}
\end{equation}
where $\varphi = (1 + \sqrt 5)/2$ is the golden ratio.  The
latitude choice $z_n = 1 - (2 n - 1)/M_a$ gives bands of equal area
$4\pi/M_a$, and the golden-angle azimuth ensures the in-band
azimuthal positions form a low-discrepancy 1D sequence (the
fractional parts of $n/\varphi$ have discrepancy
$O(\log M_a / M_a)$, the optimal rate among 1D sequences).  The
result is a tessellation of $\bbS^{2}$ in which every Voronoi cell
has area $\approx 4\pi / M_a$ and aspect ratio close to one --- the
geometric property that makes~\eqref{eq:fib_sphere_3d} the right
deterministic init for any single-shell angular code on $\bbS^{2}$.
Figure~\ref{fig:fib_3d} visualizes the codewords on $\bbS^{2}$ for
$M_a \in \{32, 64, 128, 256\}$.

\begin{figure}[h]
\centering
\includegraphics[width=0.85\linewidth]{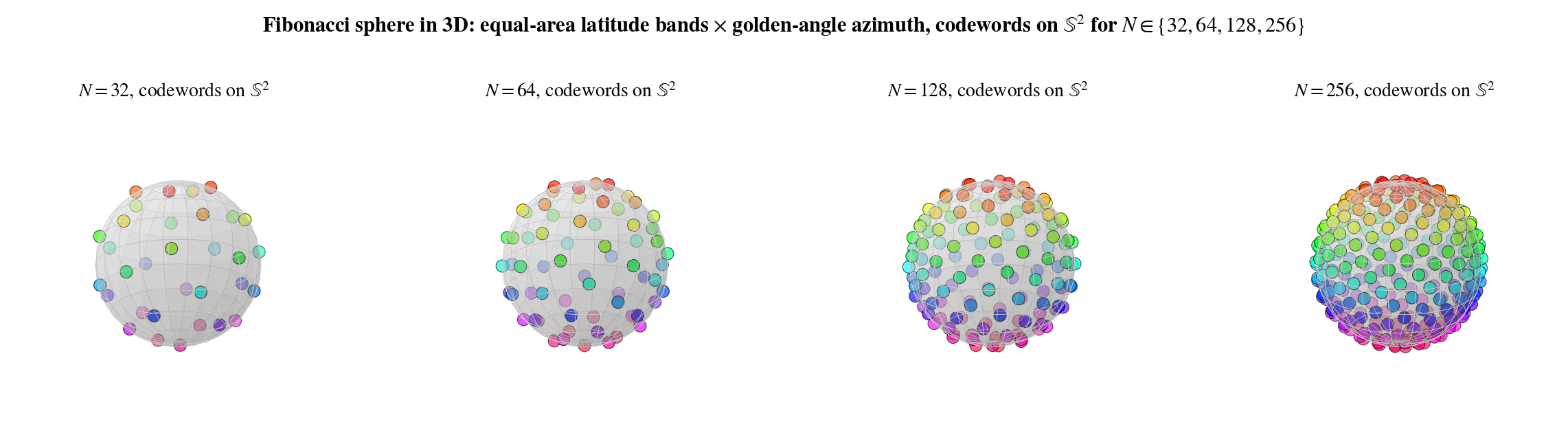}
\caption{Fibonacci sphere of $M_{a}$ codewords on $\bbS^{2}$ for
$M_{a} \in \{32, 64, 128, 256\}$, shown in orthographic projection
from a fixed viewpoint.  Equal-area latitude bands times
golden-angle azimuth give Voronoi cells of area
$\approx 4\pi/M_{a}$ and unit aspect ratio.}
\label{fig:fib_3d}
\end{figure}

\paragraph{Multi-shell extension: composing radii with directions.}
The construction in Sec.~\ref{sec:beta_quantile} places every
codeword on its own Beta-quantile radius --- that is, $N$ distinct
shells with one direction per shell.  At $k = 3$ this is wasteful
because directional packing on $\bbS^{2}$ has its own resolution
limit: the single-shell Fibonacci sphere of~\eqref{eq:fib_sphere_3d}
already tiles $\bbS^{2}$ with angular spacing $\approx 1/\sqrt{M_a}$.
We therefore parameterize the codebook as a \emph{multi-shell}
construction
\begin{equation}
\mathcal{C}(S, M_a) \;=\;
\bigcup_{s = 1}^{S}\bigl\{\,\tilde r_s\, u_n^{(s)}\,:\, n = 1, \ldots, M_a\,\bigr\},
\qquad N = S \cdot M_a,
\label{eq:multishell}
\end{equation}
with $\tilde r_s = \sqrt{\mathrm{BetaInv}(q_s;\, k/2,\, \beta_{d, k})}$
at $q_s = (s - \half)/S$ and $\{u_n^{(s)}\}_{n = 1}^{M_a}$ a fresh
Fibonacci-sphere direction set per shell.  At fixed budget
$N = S \cdot M_a$ the design problem reduces to a one-dimensional
search over the integer factorization $(S, M_a)$ that minimizes
training MSE.  In practice the optimum lies at
$S \approx \mathcal{O}(d^{1/3})$ for $k = 3$ and is found by
exhaustive enumeration of all factor pairs in $O(\sqrt N)$ time
before the Lloyd--Max polish.  At moderate $d$ this yields up to
$0.3$\,dB lower training MSE at $k = 3$ for the rate range studied
here.  At $k \geq 4$ the same multi-shell template is used with the
single-shell direction set replaced by the Roberts--Kronecker
rank-one sequence $\xi_{n, j} = \{(n - \half)\,\phi_k^{-j}\}$ where
$\phi_k$ is the positive real root of $\phi^{k + 1} = \phi + 1$.
Figure~\ref{fig:k3_multishell} visualizes the construction
explicitly at $d = 64$, $L = 4$, $M_a = 32$ ($N = 128$): each
individual shell is shown alone in the top row, and the combined
codebook is shown in 3D below.

\begin{figure}[h]
\centering
\includegraphics[width=0.81\linewidth]{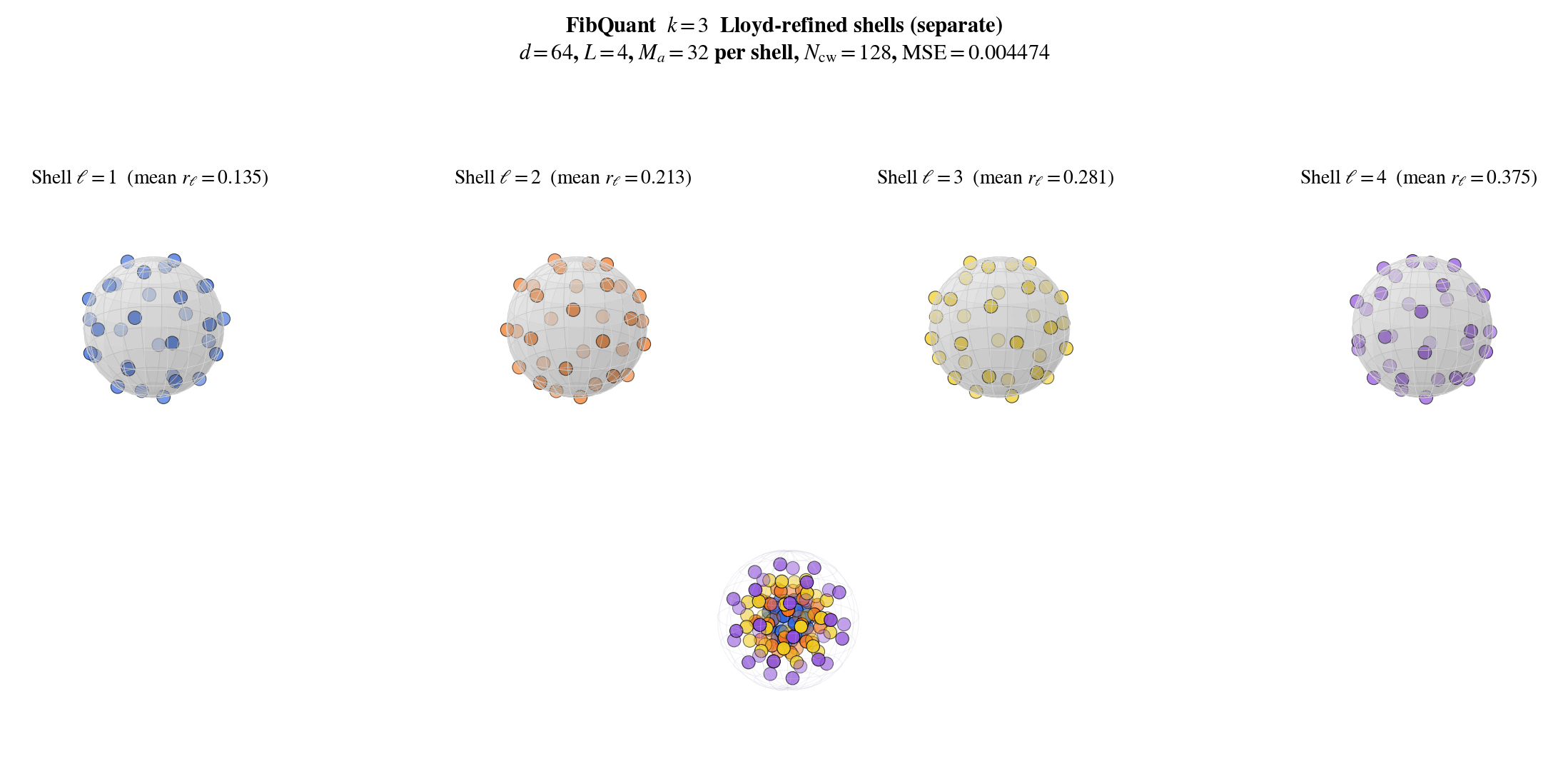}
\caption{\textsc{FibQuant} $k = 3$ Lloyd-refined multi-shell
codebook~\eqref{eq:multishell} at $d = 64$, $L = 4$ shells,
$M_a = 32$ codewords per shell ($N_{\mathrm{cw}} = 128$).
\textbf{Top row:} each shell $\ell = 1, \ldots, 4$ shown alone as
$M_a = 32$ Fibonacci-sphere codewords scaled by the per-shell
Beta-quantile radius $r_\ell$ and Lloyd-polished against Haar
samples on that shell.  \textbf{Bottom:} all four shells overlaid in
3D, with a faint wireframe at the outermost shell radius for
context.  Per-shell radii are $r_1 \approx 0.117$,
$r_2 \approx 0.201$, $r_3 \approx 0.265$, $r_4 \approx 0.330$,
sorted from inside to outside.  The ``separate'' Lloyd polish ---
each shell refined independently against samples on its own typical
shell --- recovers most of the multi-shell gain at a fraction of
the cost of joint Lloyd over all $N = 128$ codewords.}
\label{fig:k3_multishell}
\end{figure}

%----------------------------------------------------------------------
\section{Sub-linear encoder complexity for large \texorpdfstring{$N$}{N}: hierarchical list decoding}
\label{app:list_decoding}

The naive \textsc{FibQuant} encoder for one $k$-vector is one
batched dot-product against the $N$ codewords, costing $O(N k)$
FLOPs per block and $O(d N)$ per cached vector.  At the most
aggressive operating points used in the paper ($k = 64$, $N = 16384$
at $d = 64$) this is $\sim 10^{6}$ FLOPs per cached vector --- still
cheap on a GPU, but the linear-in-$N$ scaling becomes the dominant
decode-time cost when $N$ grows.  This appendix sketches two
complexity-reduction ideas that reduce the encoder cost from $O(N)$
to $O(\sqrt N)$ per $k$-block while preserving the bitstream and
the random-access decoder interface.  Decoding is unaffected --- it
remains a single table lookup per index.

\paragraph{Stage 1: radius-first search.}
The multi-shell construction~\eqref{eq:multishell} writes
$\mathcal{C} = \bigcup_{\ell = 1}^{L}\{r_\ell\, u_n^{(\ell)}\}_{n = 1}^{M_a}$
with $L$ Beta-quantile radii and $M_a = N / L$ direction codewords
per shell.  Encoding can exploit this factorization:
\begin{enumerate}
\item compute $\rho = \|y^{(m)}\|_2$ once and find
$\ell^\star = \arg\min_\ell |\rho - r_\ell|$ in $O(L)$ scalar
comparisons;
\item search the angular code on shell $\ell^\star$ alone:
$\hat n = \arg\min_n \|y^{(m)} - r_{\ell^\star} u_n^{(\ell^\star)}\|^2$
in $O(M_a) = O(N / L)$ inner products.
\end{enumerate}
The total cost is $O(L + N / L)$, minimized at $L = \sqrt N$ with
total work $\Theta(\sqrt N)$ per $k$-block.  At $(k, N) = (32,
8192)$ this is roughly $L = M_a \approx 90$ versus $N = 8192$, an
$\sim 90\times$ FLOP reduction at the encoder.  The trade-off is a
small bias: the $r_\ell$-nearest shell may not contain the
$(r u)$-nearest codeword when $\rho$ falls between two shells.  In
the rate range studied here the resulting MSE penalty is below
$0.1$\,dB on the canonical $f_{d, k}$ source.

\paragraph{Stage 2: hierarchical direction code (list decoding).}
For $k > 3$ the per-shell direction code itself can be restructured
into a two-level tree, in the spirit of classical tree-structured
vector quantization~\citep{gray1998quantization}.  Cluster the
$M_a$ direction codewords on $\bbS^{k - 1}$ into $K$ parent groups
by k-means or by a Roberts--Kronecker partition; let
$\mu_j \in \bbS^{k - 1}$ be the mean direction of group $j$, and
$\mathcal{G}_j$ the indices of its $M_a / K$ children.  Encoding a
direction $v$ proceeds in two stages:
\begin{enumerate}
\item compare $v$ against the $K$ parent centroids
$\{\mu_j\}_{j = 1}^{K}$ and keep the top-$T$ matches by inner
product, forming a candidate list $\mathcal{L}(v)$ of size
$T \cdot M_a / K$;
\item search exhaustively within $\mathcal{L}(v)$ for the nearest
sub-codeword.
\end{enumerate}
The total work is $O(K + T \cdot M_a / K)$, minimized at
$K = \sqrt{T M_a}$ with cost $\Theta(\sqrt{T M_a})$.  The list size
$T$ controls the accuracy / complexity trade-off: $T = 1$ is greedy
nearest-cluster decoding (fast, occasionally suboptimal), while
$T = K$ recovers the exhaustive search.  For
Fibonacci-sphere\,/\,Roberts--Kronecker direction sets the parent
clusters are themselves quasi-uniform on $\bbS^{k - 1}$, so the
event ``the true nearest codeword's parent is not in the top-$T$
list'' has probability that decays geometrically in $T$.

\paragraph{Composition.}
Composing both stages, the per-$k$-block cost becomes
$O(L + K + T M_a / K)$.  At $L = K = \sqrt{N / T}$ this is
$\Theta(\sqrt{T N})$ FLOPs per block, against $\Theta(N)$ for the
naive encoder.  The bitstream is unchanged: the encoder still emits
one index in $\{1, \ldots, N\}$ per block, formed as the pair
$(\ell^\star, \hat n)$ with affine offset
$\ell^\star \cdot M_a + \hat n$; random-access addressing is
preserved.  The decoder is therefore unchanged, and a downstream
consumer (e.g.\ a fused attention kernel) cannot tell whether the
encoder used the naive search or the hierarchical one.  The cost is
paid only on the cache-write path, which already runs once per
token, while the decode path --- which runs $T_{\mathrm{seq}}$
times per query and is the actual bandwidth bottleneck --- is
identical in either case.

%----------------------------------------------------------------------
\section{Per-vector reconstruction fidelity on the GPT-2 small KV cache}
\label{app:recon}

Table~\ref{tab:recon} complements the attention-output cosine
similarity of Figure~\ref{fig:gpt2_pareto} with two per-vector
reconstruction metrics, averaged across all
$12 \cdot 12 \cdot T_{\mathrm{seq}} = 73{,}728$ KV vectors of one
GPT-2 small sequence: per-vector cosine similarity
$\overline{\cos}(x, \hat x)$ and normalized reconstruction MSE
$\mathrm{NMSE} = 10\log_{10}(\|x - \hat x\|^2 / \|x\|^2)$ in dB.
Both reconstruction metrics track $\overline{\cos}_{\mathrm{attn}}$
monotonically; \textsc{TurboQuant} at $b = 4$ recovers vectors at
$\overline{\cos} = 0.996$ ($-20.4$\,dB), and \textsc{FibQuant} at
$b = 0.22$ still preserves $\overline{\cos} \approx 0.80$
($-5.7$\,dB) --- low at the vector level, but enough downstream
fidelity to keep attention outputs at $0.946$ cosine similarity
versus full precision.

\begin{table}[h]
\caption{Per-vector reconstruction fidelity on the GPT-2 small KV
cache (one sequence).  Per-coord rate $b$, per-vector
$\overline{\cos}(x, \hat x)$ ($\uparrow$), NMSE in dB
($\downarrow$), and attention-output cosine similarity
$\overline{\cos}_{\mathrm{attn}}$ ($\uparrow$).}
\label{tab:recon}
\begin{center}\small
\begin{tabular}{l l c c c c}
\toprule
Method & Operating point & $b$ & $\overline{\cos}(x, \hat x) \uparrow$
& NMSE (dB) $\downarrow$ & $\overline{\cos}_{\mathrm{attn}} \uparrow$ \\
\midrule
\textsc{TurboQuant} & $b = 4$         & $4.000$ & $0.996$ & $-20.40$ & $0.998$ \\
\textsc{TurboQuant} & $b = 3$         & $3.000$ & $0.984$ & $-14.79$ & $0.993$ \\
\textsc{TurboQuant} & $b = 2$         & $2.000$ & $0.942$ & $-9.42$  & $0.974$ \\
\midrule
\textsc{FibQuant} & $k = 2,\ N = 64$    & $3.000$ & $0.986$ & $-15.34$ & $0.994$ \\
\textsc{FibQuant} & $k = 4,\ N = 256$   & $2.000$ & $0.951$ & $-10.19$ & $0.980$ \\
\textsc{FibQuant} & $k = 8,\ N = 4096$  & $1.500$ & $0.916$ & $-8.08$  & $0.967$ \\
\textsc{FibQuant} & $k = 8,\ N = 1024$  & $1.250$ & $0.880$ & $-6.56$  & $0.955$ \\
\textsc{FibQuant} & $k = 8,\ N = 256$   & $1.000$ & $0.828$ & $-5.07$  & $0.937$ \\
\textsc{FibQuant} & $k = 16,\ N = 4096$ & $0.750$ & $0.805$ & $-5.21$  & $0.933$ \\
\textsc{FibQuant} & $k = 16,\ N = 8192$ & $0.812$ & $0.811$ & $-5.47$  & $0.940$ \\
\textsc{FibQuant} & $k = 64,\ N = 16384$& $0.219$ & $0.795$ & $-5.74$  & $0.946$ \\
\bottomrule
\end{tabular}
\end{center}
\end{table}

\end{document}